\documentclass{article}

     \usepackage[preprint,nonatbib]{neurips_2020}

\usepackage[utf8]{inputenc} 
\usepackage[T1]{fontenc}    
\usepackage{hyperref}       
\usepackage{url}            
\usepackage{booktabs}       
\usepackage{amsfonts}       
\usepackage{nicefrac}       
\usepackage{microtype}      

\usepackage[pdftex]{graphicx}
\usepackage{amsthm}

\title{Causal Discovery from Incomplete Data using An Encoder and Reinforcement Learning}

\author{%
	Xiaoshui Huang$^{1,3}$, Fujin Zhu$^{2}$, Lois Holloway$^{3}$, Ali Haidar$^{4,3}$ \\
	1. University of Sydney, 2. University of Technology Sydney \\ 3. Ingham Institute, 4. University of New Sourth Wale\\
	\texttt{Xiaoshui.Huang@sydney.edu.au; Fujinzhu.bit@gmail.com;}\\
	\texttt{Lois.Holloway@health.nsw.gov.au; A.Haidar@unsw.edu.au} \\
}

\begin{document}
	
	\maketitle
	
	\begin{abstract}
		Discovering causal structure among a set of variables is a fundamental problem in many domains. However, state-of-the-art methods seldom consider the possibility that the observational data has missing values (incomplete data), which is ubiquitous in many real-world situations. The missing value will significantly impair the performance and even make the causal discovery algorithms fail. In this paper, we propose an approach to discover causal structures from incomplete data by using a novel encoder and reinforcement learning (RL). The encoder is designed for missing data imputation as well as feature extraction. In particular, it learns to encode the currently available information (with missing values) into a robust feature representation which is then used to determine where to search the best graph. The encoder is integrated into a RL framework that can be optimized using the actor-critic algorithm. Our method takes the incomplete observational data as input and generates a causal structure graph. Experimental results on synthetic and real data demonstrate that our method can robustly generate causal structures from incomplete data. Compared with the direct combination of data imputation and causal discovery methods, our method performs generally better and can even obtain a performance gain as much as 43.2\%.
	\end{abstract}
	
	\section{Introduction}
	Causal discovery from natural phenomena is a fundamental problem across many domains, such as biology \cite{sachs2005causal}, economics \cite{pearl2009causality} and genetics \cite{peters2017elements}. Although  randomized control experiments are the gold standard for inferring causal relationships \cite{hernan2010causal}, it is impossible or very expensive in many fields such as  patient treatment \cite{whatif} and genetics \cite{peters2017elements}. Rather than designing randomized experiments, causal discovery algorithms attempt to infer causality automatically from passive observational data.  
	
	Typically, the causal relationships among a set of variables could be represented as a directed acyclic graph (DAG). To estimate the graph from observational data, there are mainly two classes of approaches in the literature \cite{triantafillou2016score}: constraint-based and score-based. Constraint-based approaches use conditional independence tests to check the existence of causal relationship between each pair of variables and try to find a graph that entails all the corresponding $d$-separations \cite{pearl2010introduction}. In contrast, score-based approaches try to find a graph $G$ that maximizes the likelihood of the data given $G$. Equivalently, they attempt to search the ''optimal'' causal graph $G$ from the directed acyclic graph space (denoted as $DAGs$) in a combinatorial optimization manner by minimizing a score function 
	\begin{eqnarray}
	\min_{G\in {DAGs}} S(G)
	\label{eq1}
	\end{eqnarray}
	
	The problem in Eq.(\ref{eq1}) is a typical NP-hard problem because the number of DAGs increase super-exponentially in the number of graph nodes \cite{chickering2004large}. Many methods have been proposed to solve this problem by designing various score functions \cite{peters2014causal, huang2018generalized}. Recently, Zheng et al. \cite{zheng2018dags} recast this combinatorial optimization problem into a continuous gradient-based optimization one which allows for nonlinear causal relationships parameterized by neural networks. Furthermore, Zhu et al. \cite{Zhu2020Causal} proposed to use a reinforcement learning agent to automatically determine the search direction in this continuous optimization problem to find the DAG with the best score.
	
	However, existing DAG searching methods, either based on combinatorial or continuous optimization, seldom consider settings with incomplete data that the observational data used for causal discovery have missing values. This phenomenon is ubiquitous in many real-world situations, which can significantly impair the final learning performance and even make the algorithms fail. For example, according to our evaluation, the algorithm in \cite{Zhu2020Causal} will not work in such a setting. To deal with the challenge of data with missing values, in this paper, we propose an approach that is able to learn the causal relationship from incomplete data.
	
	Back to the information processing procedure in our human brain, we can still obtain some basic information though the observational data has missing values. For example, if a patient is an elderly patient and does not present with a good level of fitness we might assume they are not doing any exercise. To imitate such a brain information processing procedure about incomplete data, we propose in this paper a \textit{Encoder} that uses neural networks to encode the non-missing information into a feature representation. The derived representation is then integrated into a reinforcement learning (RL) framework for searching the best causal graph. According to \cite{Zhu2020Causal},  the reason that the RL agent works well lies in that the stochastic policy can be updated properly by using the rewards. Hence, the RL agent can determine automatically where to search.
	
	The contributions of this paper are as follows:
	\begin{itemize}
		\item an RL-based approach for learning causal graphs from incomplete data.
		\item an ad-hoc encoder for extracting features from incomplete observational data for the sake of causal discovery. As a result, the whole RL framework could be optimized in an end-to-end manner while using incomplete data.
	\end{itemize}

	\section{Related Work}
	Most of the existing DAG learning algorithms could be divided into two categories: score-based and constraint-based approaches. Score-based algorithms search the space of all possible structures to maximize a decomposable score using greedy, local, or some other search algorithms. A typical example is the Greedy-Equivalent-Search (GES) algorithm \cite{chickering2002optimal,nandy2018high}. Some popular scoring functions are the Bayesian Dirichlet score\cite{de2010properties}, the Akaike Information Criterion \cite{sakamoto1986akaike}, the Bayesian Information Criterion (BIC) \cite{schwarz1978estimating} or Minimum Description Length (MDL) score \cite{chickering2002optimal}.  In contrast, constraint-based methods exploit the property of Bayesian networks that edges encode conditional dependencies. Typical examples include the well-known PC algorithm \cite{spirtes2000causation} and the FCI algorithm \cite{zhang2008completeness}. In addition, there is a suite of hybrid algorithms that combine score-based and constraint-based methods. A  prominent  example  is  the Max-Min  Hill-Climbing  (MMHC) algorithm \cite{tsamardinos2006max}. The main idea is find  the skeleton using  a  constraint-based  method  and  then orient  the  edges by a search-based method \cite{kalisch2014causal}. 
	
	As we have discussed in the previous section, since the number of DAGs increase super-exponentially in the number of graph nodes, causal discovery, as a combinatorial search problem, is NP-hard. As a result, traditional algorithms have mainly focused on small graphs and discrete variables. Recently, Zheng et al. \cite{zheng2018dags} proposes a new continuous optimization approach called NOTEARS to transform the discrete search procedure into an equality constraint. This is a pioneer work to enable the recent neural network to be applied to solve DAG learning. For example, following NOTEARS, DAG-GNN \cite{yu2019dag} uses a graph neural network to solve the DAG learning. GraN-DAG \cite{lachapelle2019gradient} extends it to deal with nonlinear relationships between variables using neural networks. 
	
	However, these methods all assume that the observational data has no missing values. This assumption may not be realistic in real applications. Observational data in many applications always has missing values \cite{josse2018introduction}. For instance, patient observational data is missing since many clinician systems think filling the form is unrelated to the treatment of illness and they may not think highly of these data, therefore, they usually do not filled the form completely. This paper aims to release this assumption to discover the causal relationships from incomplete observational data. 
	
	\begin{figure}[th]
		\centering
		\includegraphics[width=\linewidth]{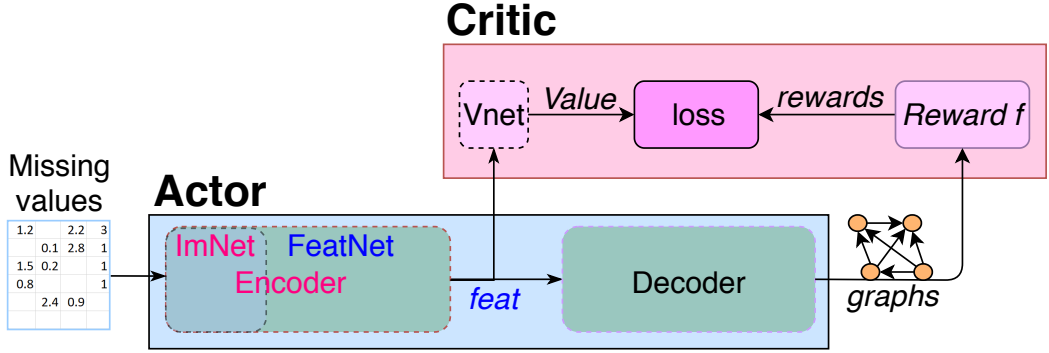}
		\caption{The proposed reinforcement learning framework of causal discovery from incomplete data. The \textbf{Actor} is an encoder-decoder neural network; the encoder consists of an imputation network (ImNet) and a feature extraction network (FeatNet). The \textbf{Critic} uses a Value network (VNet) to calculate a value for the encoded feature and calculate a reward for the decoded graphs. Both the reward and value are integrated into a loss for updating the \textit{Actor}.}
		\label{f1}
	\end{figure}
	\section{Approach}
	Our approach falls into the framework of continuous optimization based causal discovery using reinforcement learning. In this section, we firstly introduce the problem formulation and notations.  Details of the proposed actor that generates candidate causal graphs from observational data with missing values is then described. Finally, we explain how the critic helps to evaluate the generated graphs and thus provides us rewards for training the whole neural network. The overall framework of our causal discovery approach is illustrated in Figure \ref{f1}. 
	
	\subsection{Problem Formulation}
	Given an observational dataset $X \in R^{n\times d}$, where each individual unit $x_i \in R^d$ has $d$- dimensional attributes.  Each attribute $x_{ij}$ is associated with a node $j$ in a $d$-node DAG $G$, and the observed value of $x_i$ is a sample of this DAG, totally $n$ samples. To infer causal relationships among these attributes, we want to find an adjacency matrix $A \in R^{d\times d}$. Each value of the adjacency matric $A$ describes the causality between these two attributes. For example, $A_{ij}=0.5$ means the attribute $i$ has an effect to attribute $j$ and its weight is $0.5$. 
	
	Since the samples usually have missing values on the attributes, the observational data $X$ contains missing values. Missing values are toxic to the existing DAG learning methods which will make the existing optimization-based or neural network methods fail. 
	
	To cope with this challenge, we propose a reinforcement learning framework to discover the causal structure from observational data with missing values. The overall framework is illustrated in Figure \ref{f1}. In particular, given a sample of incomplete observational data, an encoder-decoder (\textit{Actor}) generates a robust feature and a graph. Then, the feature is input into a value network to compute the value and the graph input into a reward function to calculate the reward. The value and reward are utilized to calculate the loss to optimize the whole network.
	
	\subsection{Reinforcement Learning: Actor}
	\label{actor}
	In the RL-based causal discovery framework, the objective of the actor is to generate candidate causal graphs from a given observational data. In this paper, our proposed actor is an encoder-decoder neural network that is able to handle missing values in the data.
	
	\textbf{Encoder} The encoder aims to extract robust features by using the incomplete data. The proposed encoder neural network consists of two parts: the imputation network (ImNet) and the feature extraction network (FeatNet). The objective of ImNet is to impute a complete data from the current incomplete data, while the FeatNet is to extract the features from the imputation data.
	
	\textbf{Encoder: ImNet} The imputation network contains three fully-connected layers. The neurons of each layer are the same to the attribute dimension $d$ and the output is the same dimension as the input data. The first two layers use \textit{relu} as the activation function and the final layer uses the \textit{sigmod} function.
	
	Given the observational data with missing values, $X$, we define a mask matrix $M\in R^{n\times d}$ where $M_{ij}=1$ if the element $X_{ij}$ is observed and $M_{ij}=0$ otherwise. Inside the ImNet, we firstly concat $X$ with $M$, and then generate the initial output by using fully-connected layers. The main idea is that we let the neural network to automatically learn a feature from both the $X$ and $M$, and then convert the feature into an matrix of imputation data. Through this process, we obtain an initial output $X_{im}$ with the same dimension to $X$. Formally, 
	\begin{equation}
	X_{im} = ImNet(X, M)
	\end{equation}
	
	After the initial output is generated, the final complete data is calculated by  
	\begin{equation}
	X_{in} = (1-M)* X_{im} + M * X
	\end{equation}
	
	That is, the final complete data is obtained by taking the partial observation $X$ and replacing each missing value with the corresponding value of $X_{im}$. 
	
	In practice, we initialize the ImNet with pre-trained parameters for the sake of faster training. The pre-trained parameters are learned by a generative adversarial network architecture where our ImNet acts as the generator. We use the same discriminator as in \cite{yoon2018gain}. On one hand, the goal of the generator is to accurately impute missing data; on the other hand, the discriminator tries to distinguish between imputed and observed elements in the input data. They are trained iteratively in an adversarial schema.
	
	
	\textbf{Encoder: FeatNet} 
	In our encoder-decoder neural network for generating causal graphs, we adopt the \textit{Transformer} structure \cite{vaswani2017attention} for extracting robust features from the imputed data. Our empirical results indicate that the self-attention network is capable of extracting robust features from the imputed observational data. Overall, the input of the feature extracting network (FeatNet) is the imputed data matrix outputed from ImNet and its output is denoted as $feat \in R^{d\times k}$, where $k$ is the feature dimension of each attribute and $feat_i \in R^k$ is the feature of attribute node $i, (i=1,...d)$.
	
	\textbf{Decoder} 
	The decoder aims to generate graphs from the features. Inspired by \cite{Zhu2020Causal}, we use a single fully-connected layer as our decoder. The reason is that the one-layer network already contains enough ability to build graph \cite{Zhu2020Causal}. Moreover, a deep graph network will be difficult to train and the output may be indistinguishable \cite{li2018deeper}. Formally, the decoder used in this paper is
	\begin{equation}
	g_{ij}(W_1, W_2, U) = U^T tanh(W_1feat_i + W_2feat_j)
	\end{equation}
	where $W_1 , W_2 \in R ^{h\times k}, U \in  R^{h\times 1}$ are neuron matrix that need to be trained. $h$ is the dimension of neurons in the decoder and $k$ is the dimension of each feature $feat_i$ from the encoder. Then, we input $g_{ij}$ into a logistic sigmoid function $\theta(\cdot)$ to get a probability $p$ of an edge emitting from $x_i$ to $x_j$. Finally, to generate a binary adjacency matrix $A$, we sample according to a Bernoulli distribution with probability $p$. 
	
	\subsection{Reinforcement Learning: Critic}
	The job of the critic is to evaluate actions generated by the actor so that the actor can update its policy based on the evaluation score \cite{konda2000actor}. In this paper, the \textit{Critic} has two objectives: generating a value score for encoded features (\textit{Value}) and calculating a reward score for the decoded graphs (\textit{Reward}).
	
	\textbf{Vnet}  We use a neural network to calculate the value score of the encoded feature. This value score is used to calculate a discounted reward as \textit{Loss} in the following section. That is, the update of \textit{Actor} need to consider both the feature quality and graph quality. In this paper, the value neural network (Vnet) consists of two fully-connected layers. The input is the features from encoder and the output is a number that evaluate the value of the encoded features.
	\begin{equation}
	Value = fc(relu(fc(feat_v)))
	\end{equation}
	where $fc$ is the fully-connected layer, $relu$ is the non-linear activation function and $feat_v$ is the feature from the encoder neural network.
	
	\textbf{Reward function} Given a candidate graph $G$ generated by the decoder neural network, the reward function evaluates how well the graph fits the observational data. In this paper, we follow \cite{lachapelle2019gradient, zheng2018dags,Zhu2020Causal} and use the following score function:
	\begin{equation}
	S(G) = n*d*log((\sum_{i=1}^{d} RSS_i)/(n*d)) + \log n * Card(edges)
	\end{equation}
	where $\sum_i RSS_i$ is the least square loss that maximizes the likelihood for a Gaussian model used in \cite{zheng2018dags}. $Card(edges)$ is the number of edges in $G$.
	
	\textit{Acyclic constrain}: To make sure the generated graph is a DAG, we follow \cite{zheng2018dags} and also introduce the following acyclic constraint to the adjacency matrix $A$.
	\begin{equation}
	h(A):Tr(e^{A\bigodot  A}) - d = 0
	\end{equation}
	where $Tr(\cdot)$ calculate the trace, $e^M=\sum_{k=0}^{\infty}$ is the matrix exponential and $\bigodot$ is the Hadamard product.
	
	\textbf{Rewards} The final reward is calculated by incorporating the above score function and additional \textit{Acyclic constraints}. Formally, it is defined as
	\begin{equation}
	reward = -[S(G) + \lambda_1 I(G\notin DAGs) + \lambda_2 h(A)]
	\end{equation}
	where $I(\cdot)$ is the indicator function, $\lambda_1$ and $\lambda_2$ are two penalty coefficients.
	
	\textbf{Loss} To better optimize the actor, we use a discounted reward as the final reward. This discounted reward serves as a loss function for optimizing the actor.
	\begin{equation}
	Loss = \frac{1}{d}\sum_{i=1}^{d}(reward-value_i) + \lambda_3 \frac{1}{d}log^{prob}
	\end{equation}
	where $log^{prob}$ is the regularized entropy generated graph from the decoder Bernoulli distribution.
	
	\newtheorem{theorem}{Proposition}
	\begin{theorem}
		Minimizing the loss by backward optimizing the neural network is equivalent to optimizing the actor in reinforcement learning. The best actor is achieved iff 
		\begin{equation}
		\bigtriangledown Loss = 0
		\end{equation}
	\end{theorem}

	\begin{proof}
		The original actor-critic reinforcement learning is to use the score to compute the gradient and use the gradient to update its policy in an approximate gradient direction \cite{konda2000actor}.  The actor is best when the gradient iff $\bigtriangledown=0$. This paper utilizes the graph of neural network to automatically calculate the gradient $\bigtriangledown Loss$ and use backward to optimize the actor. Therefore,  iff $\bigtriangledown Loss=0$,  we obtain the best actor.
	\end{proof}

	\section{Synthetic Incomplete Data Generation}
	\label{synthetic}
	To systematically demonstrate the effectiveness of the proposed method, we use the following model to generate synthetic datasets. Consider a causal graph $G$ with $d$ node and each node is a variable $x_i$, we generate the value of $x_i$ by the following structural equation:
	\begin{equation}
	x_i := f_i(x_{pa(i)}) + n_i\ \ \ \ \  with\ \ n_i \sim N(0,1),i=[1,2,...,d]
	\label{generate}
	\end{equation}
	
	where $f_i$ is a linear or nonlinear function and the noises $n_i$ are mutually independent generated from Gaussian or non-Gaussian distributions. $x_{pa(i)}$ is a vector containing the elements $x_j$ such that there is an edge from $j$ to $i$ in the DAG $G$. The generated data is a matrix which consists of a number of row vectors $x$. Each row vector uses the same generation function in \ref{generate}.
	
	Specifically, the generative model in Eq.(\ref{generate}) belongs to standard linear-Gaussian model family \cite{peters2017elements} if all $f_i$ are linear and $n_i$ are Gaussian distribution. When the functions are linear but the noise $n_i$ are non-Gaussian, we obtain the linear non-Gaussian models described in \cite{hoyer2009nonlinear, shimizu2006linear}. In this paper, all the variables $x_i$ are scalars and it is straightforward to extend to more complex cases with a proper defined generation function. 
	
	To simulate the missing values of observational data, inspired by \cite{mayer2020missdeepcausal}, we randomly remove the attribute values for each sample with the missing probability of 20\%. The goal of causal discovery in this paper is, given the data vectors with missing values, to infer as much as possible about the generating mechanism; in particular, we seek to infer the generating graph $G$. 
	
	\section{Experiments}
	
	\subsection{Baselines and Evaluation Metrics}
	Most of the existing causal discovery methods will totally fail in settings with missing values. To deal with this problem, a straightforward idea is to impute the missing values first and then run the state-of-the-art causal discovery methods. Several experiments were conducted to compare our method with this straightforward idea. We selected GAIN as the imputation method since it has shown robust performances for missing data imputation \cite{yoon2018gain}. Overall, baseline methods we used for experimental comparison include:
	\begin{itemize}
		\item GAIN\cite{yoon2018gain}+NOTEARS \cite{zheng2018dags}. Since NOTEARS is widely used and obtains robust causal discovery results, we compare with it on incomplete data. The NOTEARS uses the default parameters as the released code. The input data is a completed data with imputation data filled in the positions of missing value.
		\item GAIN \cite{yoon2018gain}+RL+BIC2\cite{Zhu2020Causal}. Since RL+BIC2 is the current state-of-the-art causal discovery method, we also directly combine data imputation with  RL+BIC2 and compare with it on incomplete data. The parameters used are the default values.
	\end{itemize}
	
	In the proposed approach, we use pre-trained parameters to initialize the ImNet as discussed in Section \ref{actor}. The pre-trained parameters were generated from a generative adversarial architecture. Then, the input is the incomplete observational data and the whole neural networks were optimized together using the final loss from the \textit{Critic}. We used a batch size of 64 and hidden dimension of the ImNet is same to the data dimension.
	
	Three metrics were adopted to evaluate the estimated graphs:
	\begin{itemize}
		\item \textbf{False Discovery Rate (FDR)}: defined as the expected ratio of falsely discovered positive hypotheses to all those discovered. In the context of estimating the structure of a DAG, a positive hypothesis means that the estimated DAG has an same edge connection with the ground-truth DAG, and a negative hypothesis could be that the estimated DAG has an edge that does not exist in the true DAG. The FDR is the expected proportion of the falsely discovered connections to all those discovered \cite{li2009controlling}. For FDR, the lower the better.
		\item \textbf{True Positive Rate (TPR)}: defined as the expected proportion of the actual correctly discovered edges to all those discovered. For TPR, the higher the better.
		\item \textbf{Structural Hamming Distance (SHD)}: The structural Hamming distance \cite{tsamardinos2006max} is the number of edges that do not coincide in two partially directed acyclic graphs. Partially directed acyclic graph \cite{peters2015structural} is a graph with no directed cycle. Specifically, there is no pair $(k,j)$, such that there are directed paths from $k$ to $j$ and from $j$ to $k$. In this paper, we used partially directed acyclic graphs. Since the SHD considers both false negatives and false positives, a lower SHD indicates a better estimate of the causal graph. 
	\end{itemize}

	\subsection{Linear Models with Gaussian and Non-Gaussian Noises}
	Using the data generation model in Section \ref{synthetic} and following the settings of \cite{Zhu2020Causal}, we generated the datasets to test the performance on linear models with Gaussian and non-Gaussian noise. In particular, we randomly remove the values of each sample in the datasets to simulate the missing value problem. For Gaussian noise, we use a standard Gaussian distribution. For the non-Gaussian noise, we follow ICA-LiNGAM \cite{shimizu2006linear} to firstly generate samples from a Gaussian distribution and secondly use a power non-linearity to make them non-Gaussian. We generate $5000$ samples as datasets and conduct a random permutation. 
	
	\textbf{12 nodes:} Firstly, we conduct comparison experiments on a graph with node number $d=12$. In this experiment, the input of each epoch is randomly selected $n=64$ samples from the whole datasets that generated before and the total epoch number is 20000. Similar to NOTEARS \cite{zheng2018dags} and RL-BIC2 \cite{Zhu2020Causal}, we use $threshold=0.3$ to prune the estimated edges.
	
	Table \ref{t1} shows the comparison results on linear non-Gaussian and linear-Gaussian data models with missing values. Our method obtains better performance than NOTEARS and RL+BIC2. The better results means that our novel encoder integrating into the reinforcement learning framework performs better than directly combining with imputation method. The reason is that both the imputation sub-network and feature sub-network could be optimized together to learn a better feature for the incomplete data. Then, the better feature could be input into the decoder to find a better graph. The low SHD demonstrates our method can robustly generate causal graph from the incomplete data on linear models.
	
	\begin{table}[h]
		\centering
		\caption{Experimental results on incomplete data from linear non-Gaussian and linear Gaussian models with 12 graph nodes.}
		\label{t1}
		\begin{tabular}{c| c c c| c c c } 
			\hline
			&  \multicolumn{3}{c|}{Linear non-Gaussian} & \multicolumn{3}{c}{Linear-Gaussian}\\
			& FDR & TPR & SHD &  FDR & TPR & SHD \\ [0.5ex] 
			\hline\hline
			GAIN+NOTEARS & 0.486 & 0.514 & 31 & 0.5 & 0.459 & 33\\
			\hline
			GAIN+RL+BIC2 & 0.137 & 0.862 & 8  & 0.25 & 0.89 & 15 \\
			\hline
			Ours & 0.156 & 0.931 & 7 & 0.232 & 0.90 & 14   \\
			\hline
		\end{tabular}
	\end{table}

	\textbf{30 nodes:} Secondly, we conduct comparison experiments on graph node $d=30$ and the adjacency matrix is generated from Bernoulli distribution with parameter $0.2$. According to \cite{zheng2018dags,yu2019dag,Zhu2020Causal}, this edge probability choice corresponds to the fact that large graphs usually have low edge degrees in practice. In this experiment, the input of each epoch is randomly selected $n=128$ samples from the whole datasets that generated before and the total epoch number is 40000. Since this edge probability choice corresponds to the fact that large graphs usually have low edge degrees in practice, following \cite{Zhu2020Causal}, we we add to each edge a common bias term initialized to $-10$.
	
	Table \ref{t2} shows the comparison results. Directly combining imputation and causal discovery works not well when the number of graph node increase, RL+BIC2 only obtains 0.538 in TPR and 108 in SHD. NOTEARS obtains 89 in SHD which is worse than RL+BIC2 in solving large causal graph discovery. In contrast, our method obtains much better performance than both of these direct combination strategy, which improves at least 16.7\% in TPR and 32 in SHD. Compared with NOTEARS, we improve 64.1\% in TPR and 42 in SHD, which is a great performance gain.
	
	\begin{table}[h]
		\centering
		\caption{Experimental results on incomplete data from Bernoulli distribution with 30 graph nodes.}
		\label{t2}
		\begin{tabular}{c| c c c } 
			\hline
			& FDR & TPR & SHD  \\ [0.5ex] 
			\hline\hline
			GAIN+NOTEARS &0.648 &0.333 & 89\\
			\hline
			GAIN+RL+BIC2 & 0.523 & 0.807 & 79 \\
			\hline
			Ours         & 0.371& 0.974 & 47  \\
			\hline
		\end{tabular}
	\end{table}

	\subsection{Nonlinear Models with Quadratic Functions and Gaussian Process} \label{nonlinear}
	\textbf{Nonlinear model with quadratic functions}. 
	Furthermore, we evaluated the performance of our approach on data generated from non-linear models with quadratic functions. We used the same data generation process with \cite{Zhu2020Causal} and generated $5000$ samples, each has 10 attributes. To simulate missing data, we randomly remove the values for each sample with a missing probability of $20\%$. In this situation, identifiability of the true causal graph is guaranteed \cite{peters2014causal}.
	
	Experimental results are listed in Table \ref{t_nonli}. Compared with NOTEARS and RL+BIC2, our approach obtains better performance with a gain of 31\% in TPR and 6 in SHD. With a significantly lower SHD, we argue that our approach can robustly generate causal graph from incomplete data of nonlinear model with quadratic functions.
	
	\begin{table}[b]
		\centering
		\caption{Experimental results on incomplete data from nonlinear models}
		\label{t_nonli}
		\begin{tabular}{c| c c c |c c c } 
			\hline
			&  \multicolumn{3}{c|}{Quadratic function} & \multicolumn{3}{c}{Gaussian Process}\\
			& FDR & TPR & SHD & FDR & TPR & SHD   \\ [0.5ex] 
			\hline\hline
			GAIN+NOTEARS &0.625 &0.261 &23 &0.35 &0.295 &31\\
			\hline
			GAIN+RL+BIC2 & 0.33 & 0.34 & 19 & 0.071  & 0.295  & 31 \\
			\hline
			Ours & 0.25 & 0.65 & 13& 0.0857  & 0.727  & 12  \\
			\hline
		\end{tabular}
	\end{table}

	\textbf{Nonlinear model with Gaussian process}.
	We also compared our approach with other baselines on nonlinear models with Gaussian process. Using the simulation process in \cite{Zhu2020Causal}, we further generated missing values to get the observational data by randomly removing values from each sample with a missing probability of 20\%. 
	
	As we can see from Table \ref{t_nonli}, our method achieves a performance gain of 43.2\% in TPR and improves 19 in SHD. In addition, we can also conclude from the result that, by direct combination with data imputation, state-of-the-art causal discovery methods (NOTEARS and RL+BIC2) all perform poorly in nonlinear models when there are missing values. The reason is that their data imputation is a stand-alone stage which may be not capable to impute a suitable data for the causal graph discovery. Our method combines both data imputation and feature extraction into an encoder and the encoder integrate into a reinforcement learning framework to search the best DAG to fit the incomplete data. The experiments demonstrate that we can obtain much better performance than the direct combination of imputation and DAG learning methods. Moreover, the low SHD demonstrates our method can robustly generate causal graph from the incomplete data based on non-linear models.

	\subsection{Real Data}
	We also validate our proposed approach on a real-world dataset which is designed for reconstruction of causal graphs from physiologically relevant primary single cells \cite{sachs2005causal}. Data were collected after a series of inhibitory interventions and stimulatory cues. Specifically, researchers used a fixed stimulation process to stop cell reactions for $15min$, and then recorded and analyzed the effects of each condition on the intracellular signaling networks of human primary naive $T$ cell. This dataset is well accepted by the biological community and widely utilized in many domains \cite{Zhu2020Causal}. In our experiment, we consider the general perturbation that is to activate $T$ cells and induce proliferation and cytokine production. The observational data contains 853 samples and 11 attributes. According to \cite{sachs2005causal}, the ground truth causal graph has 17 edges evaluated manually with high accuracy. To simulate missing values, we randomly removed the value of each sample with a probability of 20\%. 
	
	We run the same graph prune process as the above experiment setting during the training for both RL+BIC2 and our approach. The experiment results are listed in Table \ref{t_real}. As we can see from the table, our proposed approach obtains significantly better performance than its competitors. Noticeably, the TPR results of all methods are very low for this real dataset. A possible reason is that the ground truth graph is very sparse with only 17 edges. Since the SHD considers both false negatives and false positives, the lower SHD in our result demonstrates the advantage of our approach in handling real incomplete data further indicates its ability for robustly generating causal graphs from the incomplete real data.
	
	\begin{table}[h]
		\centering
		\caption{Experimental results on the real dataset with missing values. }
		\label{t_real}
		\begin{tabular}{c| c c c} 
			\hline
			& FDR & TPR & SHD  \\ [0.6ex] 
			\hline\hline
			GAIN+NOTEARS & 0.692 & 0.353 & 18 \\ 
			\hline
			GAIN+RL+BIC2 & 0.8 & 0.118 & 18 \\
			\hline
			Ours & 0.417  & 0.411 & 13 \\
			\hline
		\end{tabular}
	\end{table}

	\section{Conclusion }
	In this paper, we proposed a causal discovery approach for learning causal graphs from observational data with missing values. In particular, we develop an ad-hoc encoder network for extracting a robust feature representation of the incomplete data, and then integrate the learned representation into a reinforcement learning framework to search for the best causal graph. The proposed approach uses deep neural networks to imitate the human brain information process on incomplete data. Experiments on both synthetic and real datasets demonstrate the effectiveness of our approach which can obtain as much as 43.2\% performance gain compared to existing state-of-the-art approaches.

	{\small
		\bibliographystyle{ieee_fullname}
		\bibliography{egbib}
	}
\end{document}